\documentclass{llncs}

\usepackage[english]{babel}
\usepackage{times}
\usepackage{amssymb}
\usepackage[usenames,dvipsnames]{color}
\usepackage{lineno}
\usepackage{graphicx}
\usepackage{float}
\usepackage{caption}
\usepackage{subcaption}
\usepackage{twoopt}
\usepackage{stmaryrd}
\usepackage{xifthen}
\usepackage{xspace}
\usepackage{url}
\usepackage[fleqn]{amsmath}

\newcommand{\as}[1]%
{\marginpar{\textcolor{Brown}{\textsc{Andreas}}}\textcolor{Brown}{[[#1]]}}
\newcommand{\ce}[1]%
{\marginpar{\textcolor{OliveGreen}{\textsc{Caroline}}}\textcolor{OliveGreen}{[[#1]]}}
\newcommand{\pvh}[1]%
{\marginpar{\textcolor{blue}{\textsc{Pascal}}}\textcolor{blue}{[[#1]]}}
\newcommand{\todo}[1]%
{\marginpar{\textcolor{red}{\textsc{ToDo}}}\textcolor{red}{[[#1]]}}

\newcommand{\GS}{\ensuremath{\mathcal{G}}}

\newcommand{\sNS}{\mathcal{N}}
\newcommand{\sES}{\mathcal{E}}
\newcommand{\sTS}{\mathcal{T}}
\newcommand{\sSS}{\mathcal{S}}
\newcommand{\sAS}{\mathcal{A}}
\newcommand{\evacPath}[1][]{\ifthenelse{\isempty{#1}}{{\ensuremath{p}_{k}}}{{\ensuremath{p}_{#1}}}}
\newcommand{\pathCap}[1][]{\ifthenelse{\isempty{#1}}{u_{\evacPath}}{u_{\evacPath[#1]}}}
\newcommand{\pathTransit}[1][]{\ifthenelse{\isempty{#1}}{t_{\evacPath}}{t_{\evacPath[#1]}}}
\newcommand{\pathLastDep}{\textsc{lastdep($\evacPath$)}}
\newcommandtwoopt{\travelPathNode}[2][][]{\ifthenelse{\isempty{#1} \OR \isempty{#2}}{\ensuremath{t_{k,n}}}{\ensuremath{t_{#1,#2}}}}
\newcommand{\horiz}{\mathcal{H}}

\newcommand{\dependentNodesSet}{\mathcal{D}}
\newcommand{\dependentNodesSetSet}{\Upsilon}
\newcommand{\safeNodeArea}[1][]{\ifthenelse{\isempty{#1}}{\ensuremath{s_k}}{\ensuremath{s_{#1}}}}

\newcommand{\startVar}[1][]{\ifthenelse{\isempty{#1}}{\textsc{start}_k}{\textsc{start}_{#1}}}
\newcommand{\taskEndVar}[1][]{\ifthenelse{\isempty{#1}}{\textsc{end}_k}{\textsc{end}_{#1}}}
\newcommand{\durVar}[1][]{\ifthenelse{\isempty{#1}}{\textsc{dur}_k}{\textsc{dur}_{#1}}}
\newcommand{\taskVar}[1][]{\ifthenelse{\isempty{#1}}{\textsc{task}_k}{\textsc{task}_{#1}}}
\newcommand{\flowVar}[1][]{\ifthenelse{\isempty{#1}}{\textsc{flow}_k}{\textsc{flow}_{#1}}}
\newcommand{\flowRateVar}[1][]{\ifthenelse{\isempty{#1}}{\lambda_{\flowVar}}{\lambda_{\flowVar[#1]}}}
\newcommand{\flowVarUB}[1][]{\ifthenelse{\isempty{#1}}{\flowVar^{\textsc{ub}}}{\flowVar[#1]^{\textsc{ub}}}}
\newcommand{\objVar}{\textsc{objective}}

\newcommand{\FontGlobals}[1]{\texttt{#1}\xspace}
\newcommand{\cumu}{\FontGlobals{cumulative}}
\newcommand{\disj}{\FontGlobals{disjunctive}}

\newcommand{\st}{\mbox{s.t.}\quad }
\newcommand{\direcdep}{\triangle} 
\newcommand{\dep}{\top} 
\newcommand{\indep}{\bot} 
\newcommand{\dom}{>}
\newcommand{\modelsim}{NEPP\xspace}%
\newcommand{\modelphased}{NPEPP\xspace}%
\newcommand{\cp}{CP\xspace}     

\newcommand{\et}[1]{\ensuremath{t_{#1}}}
\newcommand{\eu}[2][]{\ensuremath{u_{{#2}}^{{#1}}}}
\newcommand{\ef}[1]{\ensuremath{b_{#1}}}

\newcommand{\nd}[1]{\ensuremath{d_{#1}}}

\usepackage{tabularx,booktabs}

\newcolumntype{R}{>{\raggedleft\arraybackslash}X}%
\newcolumntype{S}{>{\hsize=.3\hsize}X}

\usepackage{algorithm}
\usepackage{algorithmicx}
\usepackage{algpseudocode}

\algblockx[ParForall]{ParForall}{EndParForall}%
[2]{\textbf{parallel\ forall}\ #1\ \textbf{in}\ #2\ \algorithmicdo}%
{\textbf{end forall}}
\algblockx[ParFor]{ParFor}{EndParFor}%
[2]{\textbf{parallel\ for}\ #1\ \textbf{to}\ #2\ \algorithmicdo}%
{\textbf{end for}}

\begin{document}

\mainmatter  
\title{A Constraint Programming Approach \\
  for Non-Preemptive Evacuation Scheduling}


%
\author{Caroline Even \inst{1}%
\and Andreas Schutt \inst{1,2}%
\and Pascal Van Hentenryck \inst{1,3}}
%
\institute{NICTA Optimisation Research Group, Melbourne, Australia%
\thanks{NICTA is funded by the Australian Government through the Department of Communications and the Australian Research Council through the ICT Centre of Excellence Program.}\\
\url{{firstname.lastname}@nicta.com.au} \quad
\url{http://org.nicta.com.au}
\and University of Melbourne, Victoria 3010, Australia
\and Australian National University, Canberra, Australia}

\toctitle{A Constraint Programming Approach for Non-Preemptive Evacuation Scheduling}
\tocauthor{Caroline Even, Andreas Schutt, and Pascal Van Hentenryck}
\maketitle

\begin{abstract} 
  Large-scale controlled evacuations require emergency services to
  select evacuation routes, decide departure times, and mobilize
  resources to issue orders, all under strict time
  constraints. Existing algorithms almost always allow for preemptive
  evacuation schedules, which are less desirable in practice. This
  paper proposes, for the first time, a constraint-based scheduling
  model that optimizes the evacuation flow rate (number of vehicles
  sent at regular time intervals) and evacuation phasing of widely
  populated areas, while ensuring a non-preemptive evacuation for each
  residential zone. Two optimization objectives are considered: (1) to
  maximize the number of evacuees reaching safety and (2) to minimize
  the overall duration of the evacuation. Preliminary results on a set
  of real-world instances show that the approach can produce, within a
  few seconds, a non-preemptive evacuation schedule which is either
  optimal or at most 6\% away of the optimal preemptive solution.
\end{abstract}




\section{Introduction}

Evacuation planning is a critical part of the preparation and response
to natural and man-made disasters. Evacuation planning assists
evacuation agencies in mitigating the negative effects of a disaster,
such as loss or harm to life, by providing them guidelines and
operational evacuation procedures so that they can make informed
decisions about whether, how and when to evacuate residents. In the
case of controlled evacuations, evacuation agencies instruct each
endangered resident to follow a specific evacuation route at a given
departure time. To communicate this information in a timely fashion,
evacuation planners must design plans which take into account
operational constraints arising in actual evacuations. In particular,
two critical challenges are the deployment of enough resources to give
precise and timely evacuation instructions to the endangered
population and the compliance of the endangered population to the
evacuation orders.  In practice, the control of an evacuation is
achieved through a mobilization process, during which mobilized
resources are sent to each residential area in order to give
instructions to endangered people. The number of mobilized resources determines the
overall rate at which evacuees leave. Finally, to maximize the chances
of compliance and success of a controlled evacuation, the evacuation
and mobilization plans must be easy to deploy for evacuation agencies
and should not leave, to the evacuees, uncontrolled alternative
routes that would affect the evacuation negatively.

Surprisingly, local authorities still primarly rely on expert
knowledge and simple heuristics to design and execute evacuation
plans, and rarely integrate human behavioral models in the process.
This is partly explained by the limited availability of approaches
producing evacuation plans that follow the current practice. Apart
from a few exceptions
\cite{Bish2013,Even2014,Even2015,Huibregtse2011,Pillac2013,Pillac2014} 
existing evacuation approaches rely on free-flow models which assume
that evacuees can be dynamically routed in the transportation network
\cite{Bretschneider2012,Lim2012,Richter2013}. These free-flow models
however violate a desirable operational constraint in actual evacuation
plans, i.e., the fact that all evacuees in a given residential zone
should preferably follow the same evacuation route.

Recently, a handful of studies considered evacuation plans where each
residential area is assigned a single evacuation path. These studies
define both a set of evacuation routes and a departure
schedule. Huibregtse et al. \cite{Huibregtse2011} propose a two-stage
algorithm that first generates a set of evacuation routes and feasible
departure times, and then assigns a route and time to each evacuated
area using an ant colony optimization algorithm.  In subsequent work,
the authors studied the robustness of the produced solution
\cite{Huibregtse2010}, and strategies to improve the compliance of
evacuees \cite{Huibregtse2012}.  Pillac et al. \cite{Pillac2014} first
introduced the Conflict-based Path Generation (CPG) approach which was
extended to contraflows by Even et al. \cite{Even2014}. CPG features a
master problem which uses paths for each residential node to schedule
the evacuation and a pricing problem which heuristically generates new
paths addressing the conflicts in the evacuation schedule.

These evacuation algorithms however do not guarantee that evacuees
will follow instructions. If the evacuation plan contains forks in the road,
evacuees may decide to change their evacuation routes as the
evacuation progresses.  This issue is addressed in
\cite{Andreas2009,Even2015} which propose evacuation plans without
forks. The resulting evacuation plan
can be thought of as a forest of evacuation trees where each tree is
rooted at a safe node (representing, say, an evacuation center) and
with residential areas at the leaves. By closing roads or controlling
intersections, these evacuation trees ensure the compliance of the
evacuees and avoid congestions induced by drivers slowing down at a
fork. Even et al. \cite{Even2015} produce such convergent evacuation
plans by decomposing the evacuation problem in a tree-design problem
and an evacuation scheduling problem.  Andreas and Smith
\cite{Andreas2009} developed a Benders decomposition algorithm that
selects convergent evacuation routes that are robust to a set of
disaster scenarios.

All of the approaches reviewed above allow preemption: The evacuation
of a residential area can be interrupted and restarted arbitrarily.
This is not desirable in practice, since such schedules will confuse
both evacuees and emergency services, and will be hard to enforce.
Non-preemptive schedules have been considered in
\cite{Cepolina08,Pillac2015} in rather different ways. In
\cite{Cepolina08}, a phased evacuation plan evacuates each area
separately, guaranteeing that no vehicles from different areas travel
on a same path at the same time. By definition, phased evacuation does not merge
evacuation flows, which is motivated by empirical evidence that such
merging can reduce the road network capacities.  The algorithm in
\cite{Pillac2015} is a column-generation approach for simultaneous
evacuation, i.e., evacuations where multiple paths can share the same
road segment at the same time. Each column represents the combination
of a path, a departure time, and a response curve capturing the
behavioral response of evacuees for each evacuation area.  Here the
flow rate of each evacuation area is restricted to pre-existing
response curves, and columns are generated individually for each
evacuation area. This column-generation approach requires a
discretization of the evacuation horizon.

\emph{This paper proposes, for the first time, a constraint
  programming approach to generate non-preemptive evacuation
  schedules.} It takes as input a set of evacuation routes, which are
either chosen by emergency services or computed by an external algorithm.
 The constraint-based scheduling model
associates a task with each residential area, uses decision variables
for modeling the number of evacuees (i.e., the \emph{flow}), the
number of vehicles to be evacuated per time unit (i.e., the \emph{flow
  rate}), and the starting time of the area evacuation; It also uses
cumulative constraints to model the road capacities. In addition, the
paper presents a decomposition scheme and dominance relationships that
decrease the computational complexity by exploiting the problem
structure. Contrary to \cite{Pillac2015}, the constraint-programming
model uses a decision variable for the flow rate of each evacuation
area (instead of a fixed set of values) and avoids discretizing
time. In contrast to~\cite{Cepolina08}, the constraint-programming
model allows for simultaneous evacuation while satisfying
practice-related constraints.

The constraint-programming model was applied on a real-life evacuation
case study for the Hawkesbury-Nepean region in New South Wales,
Australia. This region is a massive flood plain protected from a
catchment area (the blue mountains) by the Warra\-gamba dam. A major
spill would create damages that may reach billions of dollars and
require the evacuation of about 80,000 people.  Preliminary
experimental results indicate that the constraint-programming model
can be used to generate non-preemptive schedules that are almost
always within 5\% of the optimal preemptive schedules generated in prior
work. These results hold both for maximizing the number of evacuees
for a given time horizon and for minimizing the clearance time (i.e.,
the earliest time when everyone is evacuated). These results are
particularly interesting, given that the optimal preemptive solutions
produce evacuation plans which are far from practical.  Indeed,
Fig.~\ref{fig:fsp_flow} shows the repartition of departure times for
seven residential areas in the original HN80 instance using a
preemptive schedule produced by the algorithm in
\cite{Even2015}. Observe how departure times are widely distributed
within the scheduling horizon, indicating that the plan makes heavy
use of preemption and is virtually impossible to implement in
practice. Finally, experimental results on the phased version of the
algorithm indicate that phased evacuations are much less time
effective, and should not be the preferred method for short-notice or
no-notice evacuation.

\begin{figure}[tbp]
	\centering
	\includegraphics[width=\textwidth-12em]{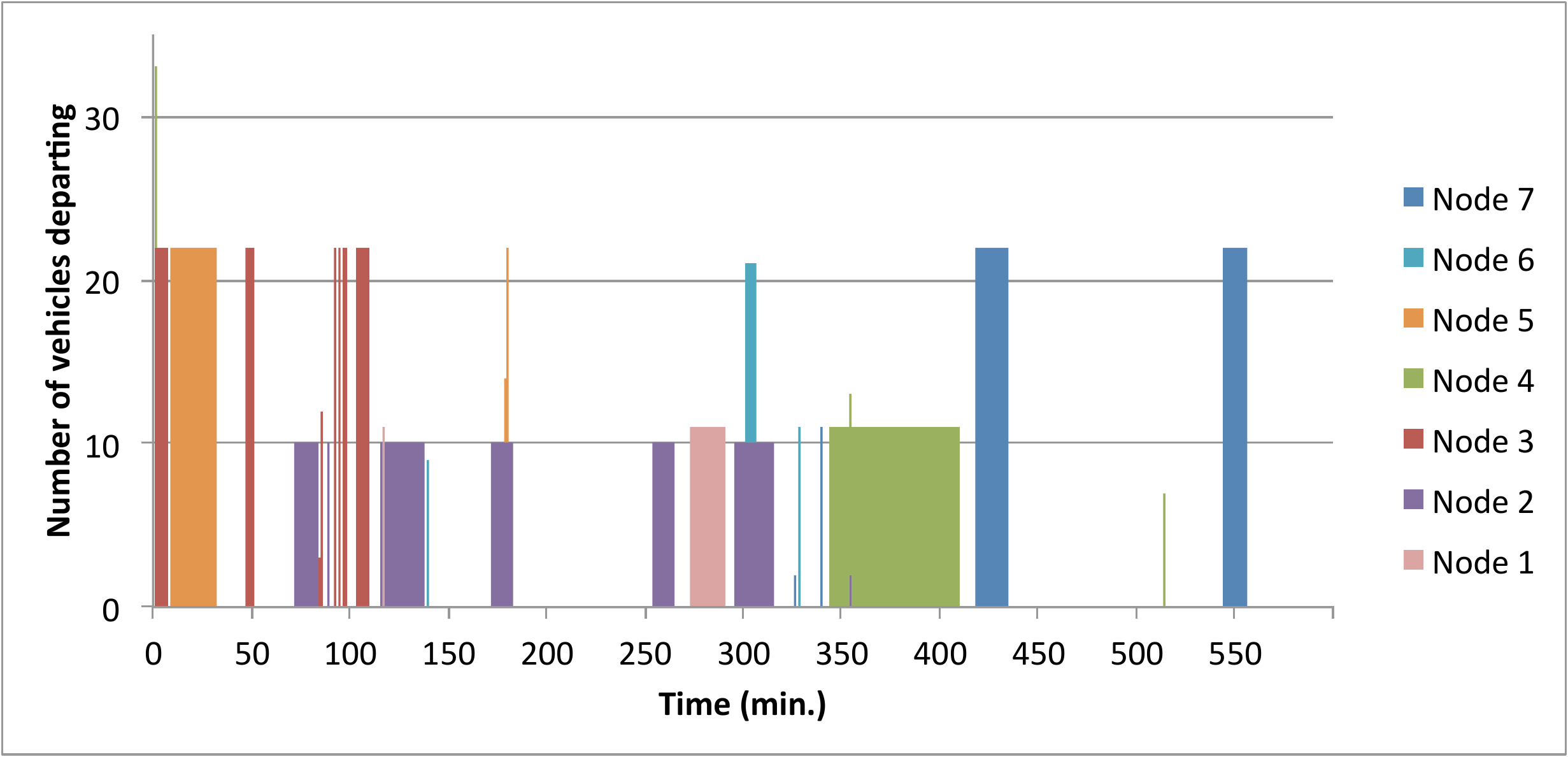}
	\caption{Departure times and flows of 7 residential areas of the HN80 instance with
          the preemptive algorithm FSP from \cite{Even2015}.}
	\label{fig:fsp_flow}
\end{figure}%

The rest of this paper is organized as follows. Section \ref{sec:def}
defines the problem. Section \ref{sec:app} presents the
constraint-programming model, including the core model, the
decomposition scheme, the dominance relationships, and the search
procedure. Section \ref{sec:results} presents the experimental
results. Section \ref{sec:conc} concludes the paper.

\section{Problem Description}
\label{sec:def}

The Evacuation Planning Problem (EPP) was introduced by the authors in
\cite{Pillac2013}. It is defined on a directed graph
$\GS=(\sNS=\sES\cup\sTS\cup\sSS, \sAS)$, where $\sES$, $\sTS$, and
$\sSS$ are the set of evacuation, transit, and safe nodes
respectively, and $\sAS$ is the set of edges. The EPP is designed to
respond to a disaster scenario, such as a flood, which may determine a
time at which some edges become unavailable.  Each evacuation node
$k\in \sES$ is characterized by a number of evacuees $\nd{k}$, while
each arc $e$ is associated with a triple $(\et{e},\eu{e},\ef{e})$,
where $\et{e} $ is the travel time, $\eu{e}$ is the capacity, and
$\ef{e}$ is the time at which the arc becomes unavailable. We denote
by $e.tail$ (resp. $e.head$) the tail (resp. head) of an edge $e$. The
problem is defined over a \emph{scheduling horizon} $\horiz$, which depends
on the disaster forecast and the time to mobilize resources.  The
objective is either (1) to maximize the total number of evacuees reaching a
safe node (for a fixed horizon) or (2) to minimize the time at which the last evacuee
reaches a safe zone (for a variable horizon). In the following, we assume that the evacuation
is carried out using private vehicles, but the proposed approach could
be adapted to other contexts, such as building evacuation.

This paper extends the EPP to the \emph{non-preemptive} 
(simultaneous) evacuation planning
problem (\modelsim) and the \emph{non-preemptive phased}
evacuation planning problem (\modelphased). Both are designed to
assist evacuation planners with the scheduling of fully controlled
evacuations. Given a set of evacuation paths, the \modelsim
decides the start time, flow, and flow rate at which vehicles are
evacuating each individual evacuation node, ensuring that the
evacuation operates without interruption.  The \modelsim
allows several evacuation nodes to use the same road segments at the
same time. In contrast, the \modelphased{} guarantees that no two
evacuation nodes use the same road segment at the same time. The
practical interest of the \modelphased{} is to evacuate designated
priority areas quickly and efficiently, by limiting the risk of any
delay caused by slowdown or traffic accidents which may result from
merging traffic.

Formally, an evacuation plan associates with each evacuation area
$k\in\sES$ exactly one \emph{evacuation path} $\evacPath$ which is used to
route all residents in $k$ to a same safe node.  Let
$\Omega_{\evacPath} = \bigcup_{k\in\sES}\evacPath$ the \emph{set of
evacuation paths} for all evacuations nodes in $\sES$. The
characteristics of a path $\evacPath$ are as follows. We denote by
$\sAS_{\evacPath}$ (resp.  $\sNS_{\evacPath}$) the \emph{set of edges}
(resp. \emph{nodes}) \emph{of} $\evacPath$ and by $\sES(e)$ the 
\emph{set of evacuation areas} 
whose path contains edge $e$, i.e., $e \in \sAS_{\evacPath}$. The \emph{travel time} $\travelPathNode$ between the
evacuation area $k$ and a node $n\in\sNS_{\evacPath}$ is equal to the
sum of the path edges travel times separating $k$ from $n$ and
$\pathTransit$ is the \emph{total travel time} between the start and end of
$\evacPath$.  The \emph{path capacity} $\pathCap$ is the minimum edge
capacity of $\evacPath$. The \emph{last possible departure time}
$\pathLastDep$ along path $\evacPath$, i.e., the latest time at which
a vehicle can depart on $\evacPath$ without being blocked, is easily
derived from all $\travelPathNode (n \in \sNS_{\evacPath})$ and the
time $\ef{e}$ at which each path edge $e \in \sAS_{\evacPath}$ becomes
unavailable. If none of the path edges $e \in \sAS_{\evacPath}$ are
cut by the disaster then $\pathLastDep = \infty$; otherwise
$\pathLastDep = \min_{e\in\sAS_{\evacPath}}(\ef{e} -
\travelPathNode[k][e.head])$. Note that the latest path departure time
only depends on the path characteristics and not on $\horiz$.

\section{The Constraint-Programming Model}
\label{sec:app}


The \modelsim and \modelphased{} are two optimization
problems whose objective is either to maximize the number of evacuees
or to minimize the overall evacuation clearance time. The key
contribution of this paper is to model them as constraint-based scheduling
problems and to use \cp to solve them. This modeling
avoids time dicretization and makes it possible to design
non-preemptive plans with variable flow rates.  This section presents
the constraint-based scheduling models, including their decision
variables, their domains, and the constraints common to both
problems. This section then presents the constraint-based scheduling
models for the \modelsim in 
Sect.~\ref{subsec:modelsim} and for the \modelphased{} in 
Sect.~\ref{subsec:cpepp}.

\subsection{Decision Variables}

The models associate with each evacuation area $k \in \sES$ the
following decision variables: the total flow of vehicles evacuated
$\flowVar$ (i.e., the number of vehicles evacuated from area $k$), the
flow rate $\flowRateVar$ representing the number of vehicles departing
per unit of time, the evacuation start time $\startVar$ (i.e., the
time at which the first vehicle is evacuated from area $k$), the
evacuation end time $\taskEndVar$, and the total evacuation duration
time $\durVar$. The last three decision variables are encapsulated
into a task variable $\taskVar$ which links the evacuation start time,
the evacuation end time, and the evacuation duration and ensures that
$\startVar + \durVar = \taskEndVar$.

The decision variables range over natural numbers. The flow
and flow rates can only be non-negative and integral since a number of vehicles is
a whole entity. The models use a time step of one minute for flow
rates and task variables which, from an operational standpoint, is a
very fine level of granularity: Any time step of finer granularity
would only be too complex to handle in practice. The domains of the
decision variables are defined as follows. The flow variable is at
most equal to the evacuation demand: $\flowVar \in [0, \nd{k}]$ where
$[a, b] = \{v \in \mathbb{N} \mid a\le v \le b\}$. The flow-rate
variable has an upper bound which is the minimum of the evacuation
demand and the path capacity rounded down to the nearest integer,
i.e., $\flowRateVar \in [1, \min(\nd{k}, \lfloor \pathCap \rfloor)]$.
The upper bounds for the evacuation start time and evacuation end time
are the smallest of the scheduling horizon minus the path travel time,
which is rounded up to the nearest integer,
and the latest path departure time, i.e., $\startVar \in
[0,\min(\horiz- \lceil \pathTransit \rceil, \lfloor \pathLastDep
\rfloor)]$.  The evacuation of an area $k$ can last at most $\nd{k}$
minutes assuming the flow rate is set to one vehicle per minute: $\durVar \in
[0,\nd{k}]$. Note that the lower bound for duration is zero in order
to capture the possibility of not evacuating the area.

\subsection{Constraints}

The \modelsim requires to schedule the flow of evacuees
coming from each evacuation area~$k$ on their respective path
$\evacPath$ such that, at any instant~$t$, the flow sent on all paths
through the network does not exceed the network edges
capacities. These flow constraints can be expressed in terms of
cumulative constraints.  Consider an edge~$e$ and the set~$\sES(e)$ of
evacuation areas whose evacuation paths use~$e$. For each evacuation
area $k \in \sES(e)$, the model introduces a new task $\taskVar^{e}$
which is a view over task $\taskVar$ satisfying:
\begin{align*}
 \startVar^e = \startVar + \travelPathNode[k][e.tail]\enspace,\quad
 \durVar^e = \durVar \enspace,\quad
 \taskEndVar^e = \taskEndVar + \travelPathNode[k][e.tail]\enspace.
\end{align*}
This new task variable accounts for the number of vehicles from evacuation
area $k$ traveling on edge $e$ at any time during the scheduling
horizon. Note that \travelPathNode[k][e.tail] is computed as the sum of the travel times
on each edge, each rounded up to the next integer for consistency with the domains of 
the decision variables. While this approximation may slightly overestimates travel times, 
it also counterbalances possible slowdowns in real-life traffic, which are not taken into 
account in this model.

The constraint-based scheduling model for the \modelsim introduces the following cumulative
constraint for edge $e$:
\begin{align*}
\cumu(\{(\taskVar^e, \flowRateVar) \mid k \in \sES(e)\}, u_e).
\end{align*}
The constraint-based scheduling model for the \modelphased{} introduces a disjunctive
constraint for edge $e$ instead:
\begin{align}
\disj(\{\taskVar^e \mid k \in \sES(e)\}). \label{disj}
\end{align}

\subsection{The Constraint-Based Scheduling Models}
\label{subsec:models}

We are now in a position to present a constraint-based scheduling
model for \modelsim-MF:
\begin{align}
\max \quad
    & \objVar= \sum_{k\in\sES}\flowVar& \label{obj:maxflow}\\
\st  \quad
    & \flowVarUB = \durVar \times \flowRateVar & \forall k \in \sES \label{ctr:scalar}\\
    & \flowVar = \min(\flowVarUB, \nd{k}) & \forall k \in \sES \label{ctr:min}\\
    & \cumu(\{(\taskVar^e, \flowRateVar) \mid k \in \sES(e)\}, u_e) & \forall e \in \sAS \label{ctr:cum}
\end{align}
The objective (\ref{obj:maxflow}) maximizes the number of evacuated
vehicles. Constraints (\ref{ctr:scalar}) and (\ref{ctr:min}) link the
flow, flow rate, and evacuation duration together, by ensuring that
the total flow for each area $k$ is the minimum of the evacuation
demand and the flow rate multiplied by the evacuation duration. They
use an auxiliary variable $\flowVarUB$ denoting an upper bound on the
number of vehicles evacuated from area $k$. Constraints (\ref{ctr:cum})
impose the capacity constraints.

The model \modelsim-SAT is the satisfaction problem version
of \modelsim-MF where the objective (\ref{obj:maxflow}) has been removed
and the constraint 
\begin{align}
\flowVar = \nd{k} &\quad \forall k \in \sES \label{ctr:fullevac}
\end{align}
has been added to ensure that every vehicle is evacuated.

To minimize clearance time, i.e., to find the minimal scheduling
horizon such that all vehicles are evacuated, it suffices to add 
the objective function to \modelsim-SAT
\begin{align}
\min \quad \objVar= \max_{k\in \sES}(\taskEndVar + \pathTransit) \label{obj:minhoriz}
\end{align}
and to relax the start and end time domains to $[0,\horiz^\textsc{ub}]$ 
where $\horiz^\textsc{ub}$ is an upper bound on the horizon required to evacuate all vehicles to a shelter.
The resulting constraint-based scheduling model is denoted by \modelsim-CT.

A constraint-programming formulation \modelphased{}-MF of the
non-preemptive phased evacuation planning problem can be obtained from
\modelsim-MF by replacing~(\ref{ctr:cum}) with~(\ref{disj}), which
prevents the flows from two distinct origins to travel on the same
edge at the same time. \modelphased{}-SAT, which is the satisfaction
problem version of \modelphased{}-MF, is obtained by removing the
objective (\ref{obj:maxflow}) and adding the constraint
(\ref{ctr:fullevac}). \modelphased{}-CT, which minimizes the
evacuation clearance time, is obtained from \modelphased{}-SAT by
adding the objective (\ref{obj:minhoriz}). Note that since the
flow-rate bounds ensure that edges capacities are always respected in
\modelphased{}, the flow-rate variable can be directly set to its
upper bound to maximize evacuation efficiency. Hence, the following
constraints are added to the \modelphased{} model:
\begin{align}
 \flowRateVar = \min(\nd{k}, \pathCap) &\quad \forall k \in \sES. \label{phased:flowRate}
\end{align}

\subsection{Problem Decomposition}

This section shows how the \modelsim{} and \modelphased{} can be decomposed by
understanding which paths compete for edges and/or how they relate to each other. In the 
following we introduce the \textit{path dependency relationship}.

\begin{definition}
  Two paths
  $\evacPath[x]$ and $\evacPath[y]$ are \emph{directly dependent},
  which is denoted by $\evacPath[x] \direcdep \evacPath[y]$, if and
  only if they share at least a common edge, i.e.,
  $\sAS_{\evacPath[x]}\cap \sAS_{\evacPath[y]} \neq \emptyset$.
\end{definition}
\begin{definition}
  Two paths $\evacPath[x]$ and $\evacPath[z]$ are \emph{indirectly
    dependent} if and only if $\neg(\evacPath[x] \direcdep
  \evacPath[z])$ and there exists a sequence of directly dependent
  paths $\evacPath[y_1], \dots, \evacPath[y_n]$ such that
  $\evacPath[x]\direcdep\evacPath[y_1]$,
  $\evacPath[y_n]\direcdep\evacPath[y_z]$ and $\evacPath[y_1]\direcdep
  \evacPath[y_{2}], \dots, \evacPath[y_{n-1}]\direcdep
  \evacPath[y_{n}]$.
\label{def:indirecdep}
\end{definition}
\begin{definition}
  Two paths $\evacPath[x]$ and $\evacPath[y]$ are \emph{dependent},
  which is denoted by $\evacPath[x] \dep \evacPath[y]$, if they are
  either directly dependent or indirectly dependent. Conversely, paths
  $\evacPath[x]$ and $\evacPath[y]$ are \emph{independent}, which is
  denoted by $\evacPath[x] \indep \evacPath[y]$, if they are neither
  directly dependent nor indirectly dependent.
\end{definition}
Obviously, the \emph{path dependency} ~$\dep$ forms an
equivalence relationship, i.e., $\dep$ is reflexive, symmetric, and
transitive.

\begin{figure}[t]
	\centering
	\begin{subfigure}[b]{0.4\textwidth}
		\includegraphics[width=\textwidth]{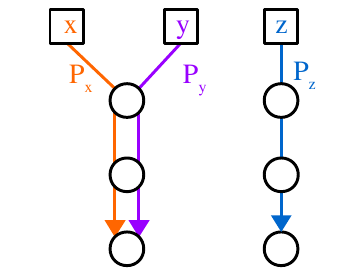}
		\caption{Directly dependent paths}
		\label{fig:paths_tree}
	\end{subfigure}%
	\begin{subfigure}[b]{0.4\textwidth}
		\includegraphics[width=\textwidth]{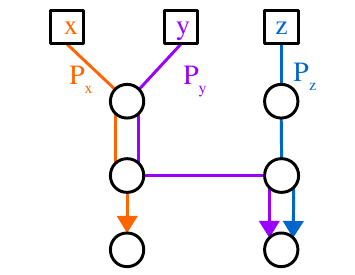}
		\caption{Indirectly dependent paths}
		\label{fig:paths_no_tree}
	\end{subfigure}%
	\caption{Illustrating independent nodes and paths.}
	\label{fig:paths}
\end{figure}

The key idea behind the decomposition is to partition the evacuation
areas $\sES$ into $\dependentNodesSetSet = \{\dependentNodesSet_0,
\dots, \dependentNodesSet_n\}$ in such a way that any two paths, respectively from
the set of evacuation areas $\dependentNodesSet_i$ and
$\dependentNodesSet_j$ ($0 \leq i < j \leq n$), are independent. As a
result, it is possible to solve the overall model by solving each
subproblem $\dependentNodesSet_i$ ($0 \leq i \leq n$)
independently and concurrently. 
Figure~\ref{fig:paths_tree} illustrates two sets of
evacuation nodes $\dependentNodesSet_0 = \{x, y\}$ and
$\dependentNodesSet_1 = \{z\}$ where paths $\evacPath[x]$ and
$\evacPath[y]$ are directly dependent and there are no indirectly
dependent paths. Figure~\ref{fig:paths_no_tree} illustrates a single
set of evacuation nodes $\dependentNodesSet = \{x, y, z\}$ where the
set of paths \{$\evacPath[x]$, $\evacPath[y]$\} and \{$\evacPath[y]$,
$\evacPath[z]$\} are directly dependent, while the set of paths
\{$\evacPath[x]$,$\evacPath[z]$\} are indirectly dependent. We now
formalize these concepts.

\begin{definition}
  Let $\Omega_{\evacPath}^{\dependentNodesSet_i}$ denote the paths of
  the set of nodes $\dependentNodesSet_i$.  Two sets $\dependentNodesSet_0$
  and $\dependentNodesSet_1$ of evacuation areas are independent if and only if
  any two paths from $\Omega_{\evacPath}^{\dependentNodesSet_0}$ and
  $\Omega_{\evacPath}^{\dependentNodesSet_1}$ respectively are
  independent. They are dependent otherwise.
\end{definition}

\begin{theorem} Let $\dependentNodesSetSet = \{\dependentNodesSet_0,
  \dots, \dependentNodesSet_n\}$ be a partition of $\sES$ such that
  $D_i$ and $D_j$ ($0 \leq i < j \leq n$) are independent. Then the
\modelsim and \modelphased{} can be solved by concatenating
the solutions of their subproblems $\dependentNodesSet_i$ ($0 \leq i \leq n)$.
\end{theorem}

\noindent
The partition $\dependentNodesSetSet = \{\dependentNodesSet_0, \dots,
\dependentNodesSet_n\}$ can be generated by an algorithm computing the
strongly connected components of a graph.  Let
$\GS_{\Omega_{\evacPath}}$ be the directed graph consisting of the
edges and vertices of all paths $\evacPath\in\Omega_{\evacPath}$ and
let $\GS^{u}_{\Omega_{\evacPath}}$ be its undirected counterpart,
i.e., the graph obtained after ignoring the direction of all edges in
$\GS_{\Omega_{\evacPath}}$. The strongly connected components of
$\GS^{u}_{\Omega_{\evacPath}}$ define the partition
$\dependentNodesSetSet$. 

\subsection{Dominance Relationships}
\label{subsec:cap}

This section shows how to exploit dominance relationships to simplify
the constraint-based scheduling models. The key idea
is to recognize that the capacity constraints of some edges are always
guaranteed to be satisfied after introducing constraints on other particular edges. 

\begin{definition}
  Let $\sAS$ a set of edges and $e, e' \in \sAS$, $e \neq e'$. Edge $e$
  \emph{dominates} $e'$, denoted by $e \dom e'$, if and only if
\begin{itemize}
\item For simultaneous evacuation, the capacity of $e$ is less than or 
    equal to the capacity of $e'$: $\eu{e} \leq \eu{e'}$;
\item The set of paths using $e'$ is a subset of the set of paths using $e$: $\sES(e')\subseteq \sES(e)$ ;
\item For non-convergent evacuation paths, the travel times for evacuation paths in $\sES(e')$ between $e$ and $e'$ are the same.
\end{itemize}
\end{definition}

\noindent
Note that two edges may be dominating each other. For this reason and without loss
of generality, this paper breaks ties arbitrarily (e.g., by selecting
the edge closer to a safe node as the dominating edge). Note also that the capacity condition is ignored for phased evacuation.

\begin{theorem}
Let $\sAS^{\dom}$ the set of dominating edges in $\sAS$. 
We can safely substitute $\sAS^{\dom}$ to $\sAS$ in (\ref{ctr:cum}) in \modelsim{}-MF such that the 
cumulative constraints are only stated for dominating edges. Similar results hold for \modelsim-CT/SAT, 
and for the disjunctive constraints in \modelphased-MF/CT/SAT.
\end{theorem}

\subsection{Additional Constraints to a Real-World Evacuation Scheduling Problem}

The flexibility of the constraint-based evacuation scheduling approach
allows to easily include many constraints appearing in real-world
evacuation scheduling. For example, each flow rate variable domain may
be restricted to a subset of values only, in order to account for the
number of door-knocking resources available to schedule the evacuation~\cite{Pillac2015}. 
Other real-world evacuation constraints may
restrict departure times in order to wait for a more accurate
prediction of an upcoming disaster or to restrict evacuation end times
to ensure that the last vehicle has left a certain amount of time
before the disaster strikes the evacuation area.

\subsection{Complexity of Phased Evacuations with Convergent Paths}

When using convergent paths, phased evacuations may be solved in
polynomial time. 

\begin{theorem}
\label{thm:phased:poly}
Model \modelphased-MF{} can be solved in polynomial time for convergent 
paths if all evacuation paths share the same latest completion time at
the safe node.
\end{theorem}
\begin{proof}[Sketch]
   Using the decomposition method and the dominance criterion,
   each subproblem with at least two evacuation paths includes exactly
   one dominating edge $e$ which is part of every evacuation path. 
   An optimal schedule can then be obtained in two steps.
   The first step builds a preemptive schedule by a sweep over the
   time starting from the minimal earliest start time of the tasks on $e$ and  
   ending before the shared completion time.
   For each point in time, it schedules one eligible task (if existing)
   with the largest flow rate (ties are broken arbitrarily) where a task 
   is eligible if the point in time is not after its earliest start time on $e$
   and it has not been fully scheduled before that time. 
   Note if a task is almost fully scheduled except the last evacuation
   batch then this step considers the actual flow rate of the last batch
   instead, which may be smaller than the task flow rate.
   This preemptive schedule is optimal as for each point in time,
   the unscheduled eligible tasks do not have an (actual) greater flow rate 
   than the scheduled ones. 
   The second step converts this optimal preemptive schedule to a 
   non-preemptive one by postponing tasks interrupting others until 
   after the interrupted tasks are completed. This transformation does 
   not change the flows and hence the non-preemptive schedule has the  
   same objective value as the preemptive schedule. \qed 
\end{proof}

\subsection{The Search Procedure}
\label{sec:search}

The search procedure considers each \textit{unassigned} task in turn
and assigns its underlying variables. A task $\taskVar$ is
\textit{unassigned} if the domain of any of its variables $\{\startVar,
\durVar, \taskEndVar, \flowVar, \flowRateVar\}$ has more than one
value. The search procedure selects an unassigned task $\taskVar$ and
then branch on all its underlying variables until they all are assigned. 
Depending on the considered problem, the models use different heuristics to (1) find the next
unassigned task and to (2) select the next task variable to branch on
and the value to assign.

For \modelphased{}, the flow rate is directly set to the maximal
value. The search strategy is determined by the
problem objective as follows. If the objective maximizes the number of
evacuees for a given scheduling horizon $\horiz$, 
the search is divided into two steps. The first step selects a task
with an unfixed duration and the highest remaining actual flow rate.
If the lower bound on duration of the task is at least two time units
less than its maximal duration then a minimal duration of the maximal
duration minus 1 is imposed and a maximal duration of the maximal
duration minus 2 on backtracking. Otherwise the search assigns
duration in decreasing order starting with the largest value in its
domain. The second step selects tasks according to their earliest
start time and assigns a start time in increasing order.\footnote{Note
  that the search procedure does not assume convergent paths or
  restrictions on the latest arrival times at the safe node.} If the
objective is to minimize the horizon $\horiz$ such that all vehicles
are evacuated, then the search selects the next unassigned task with
earliest start time by increasing order among all dominating edges,
selecting the one with maximal flow rate to break ties, and to label
the start time in increasing order.

For \modelsim, the different search heuristics are as
follow. For the choice (1), the strategy (1A) randomly selects an
unassigned task, and performs geometric restarts when the number of
backtracks equal to twice the number of variables in the model, using
a growth factor of 1.5. The strategy (1B) consists in selecting the
next unassigned task in decreasing order of evacuation demand for the
dominating edge with the greatest number of tasks. For the choice (2),
the strategy (2A) first labels the flow rate in increasing order, then
the task start time also in increasing order and, finally, the flow in
decreasing order. The strategy (2B) first labels the flow rate in
decreasing order, then the flow in decreasing order again and,
finally, the start time in increasing order.

\section{Experimental Results}
\label{sec:results}

This section reports experiments on a set of instances used in
\cite{Even2015}. These instances are derived from a real-world case
study: the evacuation of the Hawkesbury-Nepean (HN) floodplain. The HN
evacuation graph contains 80 evacuated nodes, 5 safe nodes, 184
transit nodes, 580 edges and 38343 vehicles to evacuate. The
experimental results also consider a class of instances HN80-Ix using
the HN evacuation graph but a number of evacuees scaled by a factor $x
\in \{1.1, 1.2, 1.4, 1.7, 2.0, 2.5, 3.0\}$ to model population growth. 
For simplicity, the
experiments did not consider a flood scenario and assume that network
edges are always available within the scheduling horizon $\horiz$. It
is easy to generalize the results to various flood scenarios.

For each evacuation instance, a set of convergent evacuation paths was obtained
from the TDFS approach \cite{Even2015}. The TDFS model is a MIP which
is highly scalable as it aggregates edge capacities and abstracts
time.  The paths were obtained for the maximization of the number of
evacuees within a scheduling horizon of 10 hours, allowing preemptive evacuation scheduling. Thus, for each
instance, the set of evacuation paths can be thought of as forming a
forest where each evacuation tree is rooted at a safe node and each
leaf is an evacuated node. In this particular case, each evacuation
tree is a strongly connected component. It is important to emphasize
that paths are not necessarily the same for the different HN instances, 
nor necessarily optimal for non-preemptive scheduling,
which explains some non-monotonic behavior in the
results. Table \ref{tab:instances_stats} reports the evacuation paths,
the number of vehicles to evacuate (\#vehicles), the number of
strongly connected components (\#scc), the number of evacuated nodes,
the number of vehicles per scc (scc details) for each HN80-Ix
instance. Each strongly connected component is represented by a pair
$\{x,y\}$ where $x$ is the number of evacuated nodes and $y$ the
number of vehicles to evacuate. 

\begin{table}[tbp]
\footnotesize
\caption{The strongly connected components associated with each HN80-Ix instance.}
\label{tab:instances_stats}
\begin{tabularx}{\textwidth}{lllX}
\toprule
\textbf{Instance} & \textbf{\#vehicles} & \textbf{\#scc} & \textbf{scc details}\\ 
\midrule
\textbf{HN80} & 38343 & 5 & \{22,9048\}, \{17,10169\}, \{14,6490\}, \{22,9534\}, \{5,3102\}\\ 
\textbf{HN80-I1.1} & 42183 & 4 &  \{1,751\}, \{2,1281\}, \{40,23656\}, \{37,16495\}\\ 
\textbf{HN80-I1.2} & 46009 & 3 & \{2,1398\}, \{35,18057\}, \{43,26554\}\\ 
\textbf{HN80-I1.4} & 53677 & 5 & \{2,1631\}, \{28,16737\}, \{27,19225\}, \{4,3824\}, \{19,12260\}\\ 
\textbf{HN80-I1.7} & 65187 & 4 & \{22,16992\}, \{2,1980\}, \{42,33240\}, \{14,12975\} \\ 
\textbf{HN80-I2.0} & 76686 & 4 & \{15,13974\}, \{38,40612\}, \{2,2330\}, \{25,19770\} \\ 
\textbf{HN80-I2.5} & 95879 & 5 & \{32,36260\}, \{6,11214\}, \{16,17983\}, \{6,9324\}, \{20,21098\}\\ 
\textbf{HN80-I3.0} & 115029 & 5 & \{5,11574\}, \{12,14184\}, \{19,23403\}, \{7,13068\}, \{29,39651\}\\ 
\bottomrule
\end{tabularx}
\end{table}

The experimental results compare the flow scheduling results obtained
with the \modelsim{} and \modelphased{} approaches and the flow
scheduling problem (FSP) formulation presented in \cite{Even2015}. The
FSP is solved using a LP and it relaxes the non-preemptive
constraints. Indeed, the flow leaving an evacuated node may be
interrupted and restarted subsequently at any time $t\in\horiz$,
possibly multiple times. Moreover, the flow rates in the FSP algorithm
are not necessarily constant over time, giving substantial scheduling
flexibility to the FSP but making it very difficult to implement in
practice. Once again, the FSP comes in two versions. The objective of
the core FSP is to maximize the number of vehicles reaching safety,
while the objective of FSP-CT is to minimize the evacuation clearance
time. In order to compare the FSP algorithm and the
constraint-programming approaches of this paper fairly, the FSP is
solved with a time discretization of 1 minute. The experiments for the
\modelsim{} and \modelphased{} models use different search heuristics
and each experimental run was given a maximal runtime of 1800 seconds
per strongly connected component. The results were obtained on 64-bits
machines with 3.1GHz AMD 6-Core Opteron 4334 and 64Gb of RAM and the
scheduling algorithms were implemented using the programming language
JAVA 8 and the constraint solver Choco 3.3.0, except for \modelphased-MF where the 
search was implemented in ObjectiveCP.

\paragraph{Maximizing the Flow of Evacuees.}

Table \ref{tab:simultaneous-mf} compares, for each HN80-Ix instance
and a 10-hour scheduling horizon, the percentage of vehicles evacuated
(Perc. Evac.) and the solving time in seconds (CPU (s)) with FSP,
\modelsim-MF/SAT and \modelphased-MF/SAT. All solutions found with FSP are optimal and are thus
an upper bound on the number of vehicles that can be evacuated with
\modelsim and \modelphased. Prior to solving
\modelsim-MF (resp. \modelphased-MF), the algorithm attempts to
solve \modelsim-SAT (resp. \modelphased-SAT) with a 60s time limit
and, when this is successful, the annotation (SAT) is added next to
the percentage of vehicles evacuated. As we make use of decomposition and parallel computing, the
 reported CPU for NEPP/NPEPP is the latest of the time 
at which the best solution is found among all strongly connected components. 
The table reports the best results across the heuristics,
i.e., the run where the most vehicles are evacuated ; for the random strategy,
 the best result is reported across 10 runs (note that the standard deviation 
 for the objective value ranges between $0.4\%$ and $1.1\%$ only across all instances). The search
strategy for the best run is shown in column (Search) as a combination
\{TaskVar, VarOrder\} where TaskVar is the heuristic for choosing the
next task variable and VarOrder is the heuristic for labeling the
task variables.
\begin{table}[tbp]
\footnotesize
\caption{Percentage of Vehicles Evacuated with FSP, NEPP-MF/SAT, NPEPP-MF/SAT.}
\label{tab:simultaneous-mf}
\begin{tabularx}{\linewidth}{XXXXXXXX}
\toprule
 & \multicolumn{2}{c}{\textbf{FSP}}  & \multicolumn{3}{c}{\textbf{NEPP-MF/SAT}} & \multicolumn{2}{c}{\textbf{NPEPP-MF/SAT}}  \\ %
\cmidrule(r){2-3}\cmidrule(r){4-6}\cmidrule(r){7-8}
\textbf{Instance} & \textbf{CPU (s)} & \textbf{Perc. Evac.} & \textbf{CPU (s)} & \textbf{Perc. Evac.} & \textbf{Search} & \textbf{CPU (s)} & \textbf{Perc. Evac.} \\ 
\midrule
\textbf{HN80} & 0.9 & 100.0\% & 0.2 & 100.0\% (SAT) & \{1B, 2B\} & 3.4 & 96.9\% \\ 
\textbf{HN80-I1.1} & 1.1 & 100.0\% & 1538.9 & 99.2\% & \{1A, 2B\} & 1295.9 & 58.4\% \\ 
\textbf{HN80-I1.2} & 1.0 & 100.0\% & 0.4 & 100.0\% (SAT) & \{1B, 2B\} & 1444.4s & 57.7\% \\ 
\textbf{HN80-I1.4} & 1.3 & 100.0\%  & 1347.5 & 99.3\% & \{1A, 2B\} & 307.0 & 73.0\% \\ 
\textbf{HN80-I1.7} & 1.8 & 100.0\% & 1374.9 & 97.8\% & \{1A, 2A\} & 0.3 & 59.0\% \\ 
\textbf{HN80-I2.0} & 2.0 & 97.9\%  & 1770.1 & 93.1\% & \{1A, 2B\} & 5.9 & 52.8\% \\ 
\textbf{HN80-I2.5} & 1.8 & 82.2\%  & 1664.1 & 79.2\% & \{1A, 2B\} & 0.1 & 51.5\% \\ 
\textbf{HN80-I3.0} & 1.4 & 69.2\% & 887.2 & 67.5\% & \{1A, 2B\} & 0.1 & 43.1\% \\ 
\bottomrule
\end{tabularx}
\end{table}

The results highlight the fact that the constraint-based simultaneous
scheduling model finds very high-quality results. On the first five
instances, with population growth up to 70\%, the solutions of
\modelsim-MF are within 2.2\% of the preemptive bound. This is also
the case for the largest instance. In the worst case, the
constraint-based scheduling model is about 4.9\% away from the
preemptive lower bound. It is thus reasonable to conclude that the
constraint-based algorithms may be of significant value to emergency
services as they produce realistic plans for large-scaled controlled
evacuations.  For \modelphased-MF, the solver found optimal solutions
and proved optimality for all instances, except HN80-I1.1 and
HN80-I1.2 for which the best found solution was within $0.1\%$ of the
optimal one.\footnote{In our experiments, the problem \modelphased
  satisfies the condition for Theorem~\ref{thm:phased:poly}. Thus,
  these instances can be solved almost instantly using the algorithm
  outlined in the proof of Theorem~\ref{thm:phased:poly}.} The results
indicate that a phased evacuation is much less effective in practice
and decreases the number of evacuees reaching safety by up to 40\% in
many instances.  Unless phased evacuations allow an evacuation of all
endangered people, they are unlikely to be applied in practice, even
if they guarantee the absence of traffic merging.

Figure \ref{fig:solution_time} shows how the quality of solutions
improves over time for all HN-Ix instances which are not completely
evacuated, for a particular run. For all instances, a high-quality solution is found within
10 seconds, which makes the algorithm applicable to a wide variety of
situations. When practical, giving the algorithm more time may still
produce significant benefits: For instance, on HN-I1.7 the percentage
of vehicles increases from 93.0\% to 97.6\% when the algorithm is
given 1800s. Such improvements are significant in practice since they
may be the difference between life and death.

\begin{figure}[tbp]
\centering
\includegraphics[width=\textwidth-12em]{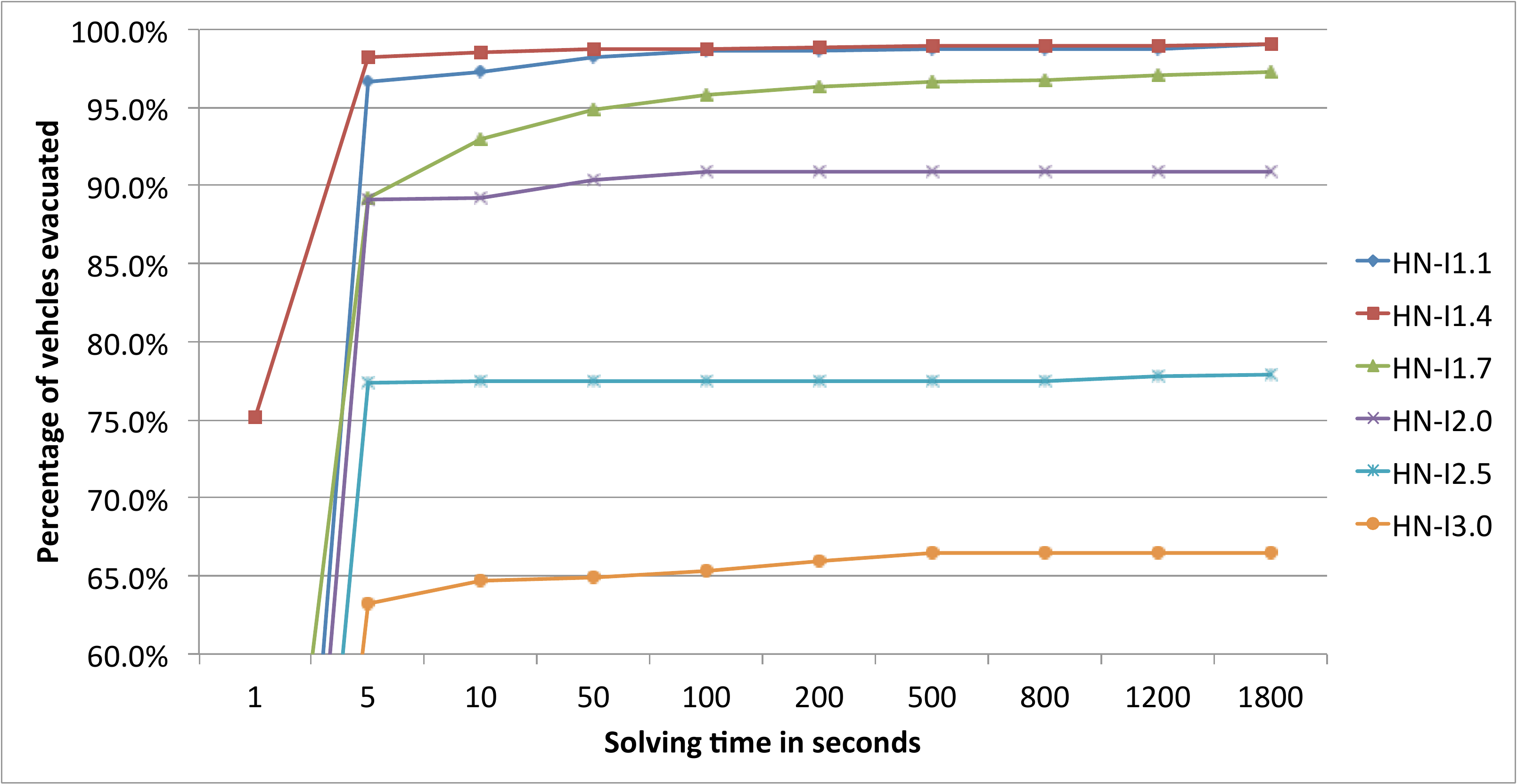}
\caption{Quality of solutions over Time for NEPP-MF.}
\label{fig:solution_time}
\end{figure}%

%

\paragraph{Profile of the Evacuation Schedules.}

To demontrate the benefits of \modelsim{}, it is useful to look at the
evacuation profiles produced by the various algorithms. Recall that
Fig.~\ref{fig:fsp_flow} displays a repartition of departure times
for seven evacuated nodes in the original HN80 instance in the optimal
solution produced by the FSP solver. \textit{The key observation is
  that, for several residential areas, the departure times are widely
  distributed within the scheduling horizon, indicating that the FSP
  solution makes heavy use of preemption.} In the FSP solution, the
number of vehicles departing at each time step is often equal to the
path capacity. But there are also some suprising combinations
\{evacuated node, time step\}, such as \{3, 50\}, \{3, 84\} and \{3,
85\} where the flow rate is respectively 22, 3, and 12 for evacuation
area 3.  In summary, the FSP solution is unlikely to be the basis of a
controlled evacuation: It is just too difficult to enforce such a
complicated schedule. Figure \ref{fig:cf_flow} shows a repartition of departure times for
the same nodes in the original HN80 instance using the \modelsim{}. 
The evacuation profile for the departure
times is extremely simple and can easily be the basis of a controlled
evacuation.  Its simplicity contrasts with the complexity of the FSP
solution and demonstrates the value of the constraint-programming
approach promoted in this paper.

\begin{figure}[tbp]
	\centering
	\includegraphics[width=\textwidth-12em]{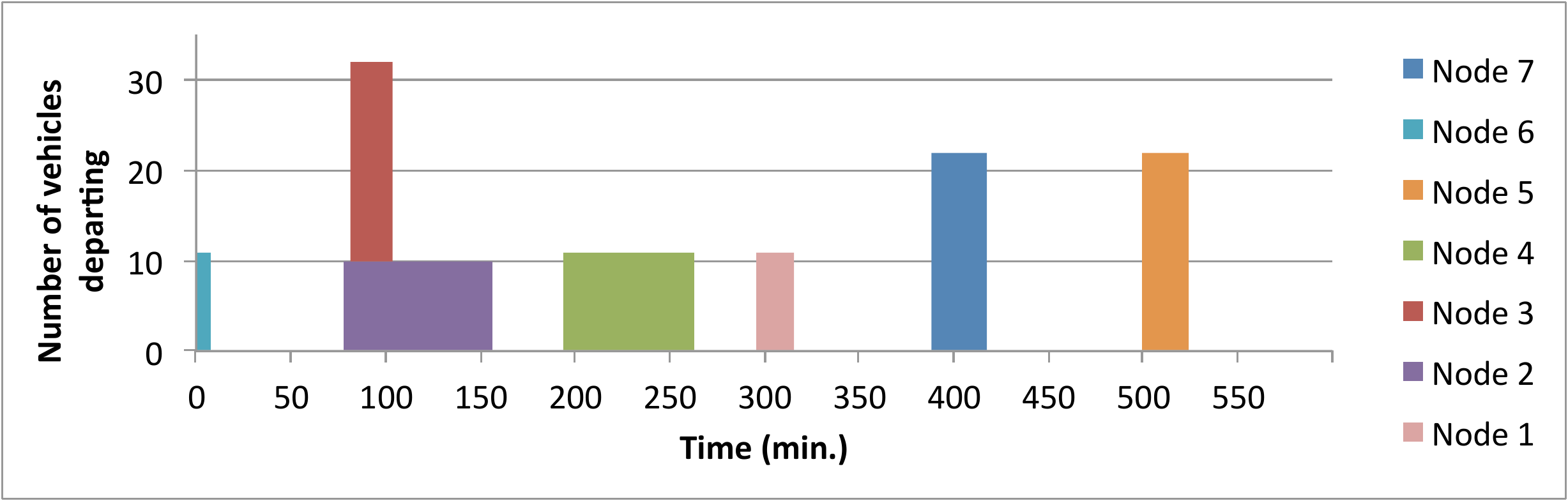}
	\caption{Departure times and flows of 7 evacuated nodes for the \modelsim{} solver.}
	\label{fig:cf_flow}
\end{figure}%

\paragraph{Minimizing the Clearance Time.}

Table \ref{tab:min_horiz} compares, for each HN80-Ix instance, the
minimal clearance time in minutes (CT (min)) found with FSP-CT,
\modelsim{}-CT and \modelphased{}-CT. All solutions found with FSP-CT
and \modelphased{}-CT are optimal for the given set of paths. Once
again, solutions found by \modelsim{}-CT are of high-quality and
reasonably close to the preemptive lower bound produced by FSP-CT. In
the worst case, the results are within 5.1\% of the preemptive lower
bound. The clearance times of the phased evacuations, which are
optimal, are significantly larger than for the \modelsim{}. Note again
that paths are different between instances and are not necessarily
optimal with respect to different scheduling horizons, which explain
inconsistencies such as the horizon found for HN80-I1.4 being shorter
than the horizon found for HN80-I1.2 with \modelphased{}-CT.

\begin{table}[tbp]
\footnotesize
\caption{Evacuation clearance time (CT) with FSP-CT, \modelsim{}-CT, and \modelphased{}-CT.}
\label{tab:min_horiz}
\begin{tabularx}{\linewidth}{XXXXXXXX}
\toprule
 & \multicolumn{2}{c}{\textbf{FSP-CT}} & \multicolumn{3}{c}{\textbf{\modelsim-CT}}  & \multicolumn{2}{c}{\textbf{\modelphased{}-CT}} \\ %
\cmidrule(r){2-3}\cmidrule(r){4-6}\cmidrule(r){7-8}
\textbf{Instance} & \textbf{CPU (s)} & \textbf{CT (min)} & \textbf{CPU (s)} & \textbf{CT (min)} & \textbf{Search} & \textbf{CPU (s)} & \textbf{CT (min)} \\ 
\midrule
\textbf{HN80} & 4.3 & 398 & 1370.9 & 409 & \{1B, 2A\} & $ 0.2 $ & 680  \\ 
\textbf{HN80-I1.1} & 7.5 & 582 & 280.3 & 616 & \{1A, 2A\} & $ 0.3 $ & 1716   \\ 
\textbf{HN80-I1.2} & 5.0 & 577 & 6.2 & 590 & \{1A, 2B\} & $ 0.3 $ & 1866 \\ 
\textbf{HN80-I1.4} & 6.1 & 587 & 1386.0 & 614 & \{1A, 2A\} & $ 0.4 $ & 1226  \\ 
\textbf{HN80-I1.7} & 7.3 & 583 & 1298.6 & 610 & \{1A, 2A\} & $ 0.3 $ & 2307 \\ 
\textbf{HN80-I2.0} & 4.0 & 625 & 1713.7 & 657 & \{1A, 2B\} & $ 0.4 $ & 2909 \\ 
\textbf{HN80-I2.5} & 8.4 & 1092 & 110.6 & 1133 & \{1A, 2B\} & $ 0.4 $ & 1884 \\ 
\textbf{HN80-I3.0} & 9.4 & 1232 & 212.6 & 1235 & \{1A, 2B\} & $ 0.3 $ & 2467 \\ 
\bottomrule
\end{tabularx}
\end{table}

\paragraph{The Impact of the Flow Rates.} 

The constraint-based scheduling models have the flow rates as decision
variables, which increases the flexibility of the solutions. Table
\ref{tab:simultaneous-ratectr-mf} studies the benefits of this
flexibility and compares the general results with the case where the
flow rates must be selected from a specific set, here $\{2, 6, 10, 15,
20\}$. This is similar to the approach proposed in \cite{Pillac2015}, which uses a fixed set of
response curves and their associated mobilization resources. Note that
the column-generation algorithm in \cite{Pillac2015} does not produce convergent plans and
discretizes time. The results seem to indicate that flexible flow
rates sometimes bring benefits, especially for the larger instances where
the benefits can reach 3.0\% ; nonetheless the possible loss when using fixed rates is not 
substantial and may capture some practical situations.

\begin{table}[tbp]
\footnotesize
\caption{Vehicles Evacuated with NEPP-MF with Flow Rates in $\{2, 6, 10, 15, 20\}$.}
\label{tab:simultaneous-ratectr-mf}
\begin{tabularx}{\linewidth}{XXXX}
\toprule
\textbf{Instance} & \textbf{CPU (s)} & \textbf{Perc. Evac.} & \textbf{Search} \\ 
\midrule
\textbf{HN80} & 0.4 & 100.0\% (SAT) & \{1B, 2B\} \\ 
\textbf{HN80-I1.1} & 1538.9 & 99.2\% & \{1A, 2B\}   \\ 
\textbf{HN80-I1.2} & 0.9 & 100.0\% (SAT) & \{1B, 2B\} \\ 
\textbf{HN80-I1.4} & 986.0 & 99.5\% & \{1A, 2B\} \\ 
\textbf{HN80-I1.7} & 1289.5 & 97.1\% & \{1A, 2A\} \\ 
\textbf{HN80-I2.0} & 1614.3 & 91.0\% & \{1A, 2B\} \\ 
\textbf{HN80-I2.5} & 1784.9 & 77.0\% & \{1A, 2B\} \\ 
\textbf{HN80-I3.0} & 1558.7 & 65.6\% & \{1A, 2B\} \\ 
\bottomrule
\end{tabularx}
\caption{Comparison of FSP and \modelsim{} problem sizes.}
\label{tab:problem_size}
\begin{tabularx}{\linewidth}{XXXXXXX}
\toprule
 & \multicolumn{2}{c}{\textbf{FSP-10}} & \multicolumn{2}{c}{\textbf{FSP-15}}  & \multicolumn{2}{c}{\textbf{\modelsim-MF}} \\ %
\cmidrule(r){2-3}\cmidrule(r){4-5}\cmidrule(r){6-7}
\textbf{Instance} & \textbf{\#cols} & \textbf{\#rows} & \textbf{\#cols} & \textbf{\#rows} & \textbf{\#vars} & \textbf{\#ctrs} \\ 
\midrule
\textbf{HN80} & 44651 & 145880 & 68651 & 218780 & 1958 & 2288  \\ 
\bottomrule
\end{tabularx}
\end{table}

\paragraph{Comparison of Model Sizes.}

One of the benefits of the constraint-based scheduling models is that
they do not discretize time and hence are highly scalable in memory
requirements. This is important for large-scale evacuations which may
be scheduled over multiple days. Table \ref{tab:problem_size} compares
the FSP problem size for a scheduling horizon of 10 hours (FSP-10) and
15 hours (FSP-15) with the \modelsim{}-MF problem size for
the HN80 instance, when using 1 minute time steps. It reports the
number of columns (\#cols) and the number of rows (\#rows) of the FSP
model, as well as the number of variables (\#vars) and the number of
constraints (\#ctrs) of the \modelsim{} model. As can be
seen, the number of variables and constraints grow quickly for the FSP
model and are about 2 orders of magnitude larger than those in the
\modelsim{}-MF model which is time-independent.


\section{Conclusion}
\label{sec:conc}

This paper proposes, for the first time, several constraint-based
models for controlled evacuations that produce practical and
actionable evacuation schedules. These models address several limitations
of existing methods, by ensuring non-preemptive scheduling and
satisfying operational evacuation constraints over mobilization resources.  The
algorithms are scalable, involve no time discretization, and are
capable of accommodating side constraints for specific disaster
scenarios or operational evacuation modes. Moreover, the models have
no restriction on the input set of evacuation paths, which can be convergent or
not. Preliminary experiments show that high-quality solutions, with an objective 
value close to optimal preemptive solutions objectives, can be
obtained within a few seconds, and improve over time.
Future work will focus on improving the propagation strength of the 
cumulative constraint for variable durations, flows, and flow rates, and 
on generalizing the algorithm for the joint evacuation planning and scheduling,


\clearpage

\bibliography{cp2015}

\begin{thebibliography}{10}
\providecommand{\url}[1]{\texttt{#1}}
\providecommand{\urlprefix}{URL }

\bibitem{Andreas2009}
Andreas, A.K., Smith, J.C.: Decomposition algorithms for the design of a
  nonsimultaneous capacitated evacuation tree network. Networks  53(2),
  91--103 (2009)

\bibitem{Bish2013}
Bish, D.R., Sherali, H.D.: Aggregate-level demand management in evacuation
  planning. European Journal of Operational Research  224(1),  79--92 (2013)

\bibitem{Bretschneider2012}
Bretschneider, S., Kimms, A.: Pattern-based evacuation planning for urban
  areas. European Journal of Operational Research  216(1),  57--69 (2012)

\bibitem{Cepolina08}
Cepolina, E.M.: Phased evacuation: An optimisation model which takes into
  account the capacity drop phenomenon in pedestrian flows. Fire Safety Journal
  (44),  532--544 (2008)

\bibitem{Even2014}
Even, C., Pillac, V., Van~Hentenryck, P.: Nicta evacuation planner: Actionable
  evacuation plans with contraflows. In: Proceedings of the 20th European
  Conference on Artificial Intelligence (ECAI 2014). Frontiers in Artificial
  Intelligence and Applications, vol. 263, pp. 1143--1148. IOS Press, Amsterdam
  (2014)

\bibitem{Even2015}
Even, C., Pillac, V., Van~Hentenryck, P.: Convergent plans for large-scale
  evacuations. In: Proceedings of the 29th AAAI Conference on Artificial
  Intelligence (AAAI-15) (2015), in press

\bibitem{Huibregtse2010}
Huibregtse, O.L., Bliemer, M.C., Hoogendoorn, S.P.: Analysis of near-optimal
  evacuation instructions. Procedia Engineering  3,  189--203 (2010), 1st
  Conference on Evacuation Modeling and Management

\bibitem{Huibregtse2012}
Huibregtse, O.L., Hegyi, A., Hoogendoorn, S.: Blocking roads to increase the
  evacuation efficiency. Journal of Advanced Transportation  46(3),  282--289
  (2012)

\bibitem{Huibregtse2011}
Huibregtse, O.L., Hoogendoorn, S.P., Hegyi, A., Bliemer, M.C.J.: A method to
  optimize evacuation instructions. OR Spectrum  33(3),  595--627 (2011)

\bibitem{Lim2012}
Lim, G.J., Zangeneh, S., Baharnemati, M.R., Assavapokee, T.: A capacitated
  network flow optimization approach for short notice evacuation planning.
  European Journal of Operational Research  223(1),  234--245 (2012)

\bibitem{Pillac2013}
Pillac, V., Even, C., Van~Hentenryck, P.: A conflict-based path-generation
  heuristic for evacuation planning. Tech. Rep. VRL-7393, NICTA (2013),
  arXiv:1309.2693, submitted for publication

\bibitem{Pillac2015}
Pillac, V., Cebrian, M., Van~Hentenryck, P.: A column-generation approach for
  joint mobilization and evacuation planning. In: International Conference on
  Integration of Artificial Intelligence and Operations Research Techniques in
  Constraint Programming for Combinatorial Optimization Problems (CPAIOR).
  Barcelona (may 2015)

\bibitem{Pillac2014}
Pillac, V., Van~Hentenryck, P., Even, C.: A path-generation matheuristic for
  large scale evacuation planning. In: Blesa, M., Blum, C., Voss, S. (eds.)
  Hybrid Metaheuristics. Lecture Notes in Computer Science, vol. 8457, pp.
  71--84. Springer (2014), 9th International Workshop on Hybrid Metaheuristics

\bibitem{Richter2013}
Richter, K.F., Shi, M., Gan, H.S., Winter, S.: Decentralized evacuation
  management. Transportation Research Part C: Emerging Technologies  31,  1--17
  (2013)

\end{thebibliography}

\end{document}